\documentclass{article}

\usepackage[utf8]{inputenc} 
\usepackage{amsmath}
\usepackage{amsthm}
\usepackage{amssymb,bm, cases, mathtools, thmtools}
\usepackage{verbatim}
\usepackage{graphicx}\graphicspath{{figures/}}
\usepackage{multicol}
\usepackage{multirow}
\usepackage{tabularx}
\usepackage[usenames,dvipsnames]{xcolor}
\usepackage{mathrsfs} 
\usepackage{url}
\usepackage{enumitem}

\usepackage{pifont}
\usepackage{amssymb}

\usepackage[T1]{fontenc}    
\usepackage{booktabs}       
\usepackage{amsfonts}       
\usepackage{nicefrac}       
\usepackage{microtype}      

\usepackage{subfigure}
\usepackage{chngpage}
\usepackage{algorithm}
\usepackage{algorithmicx}
\usepackage{algpseudocode}
\usepackage[font=small,labelfont=bf]{caption}
\usepackage{color}
\usepackage{wrapfig}
\usepackage{xspace}
\usepackage{fancyhdr}
\usepackage{tikz,lipsum,lmodern}
\usepackage[most]{tcolorbox}

\usepackage{natbib}

\usepackage[colorlinks,citecolor=blue,urlcolor=blue,linkcolor=RawSienna]{hyperref}
\usepackage{datetime}

\DeclareMathAlphabet\EuRoman{U}{eur}{m}{n}
\SetMathAlphabet\EuRoman{bold}{U}{eur}{b}{n}

\usepackage[capitalize]{cleveref}

\usepackage{makecell}

\theoremstyle{plain}
\newtheorem{theorem}{Theorem}[section]

\theoremstyle{definition}

\theoremstyle{remark}

\makeatletter
\let\reftagform@=\tagform@
\def\tagform@#1{\maketag@@@{\ignorespaces\textcolor{gray}{(#1)}\unskip\@@italiccorr}}
\renewcommand{\eqref}[1]{\textup{\reftagform@{\ref{#1}}}}
\makeatother

\def\[#1\]{\begin{align}#1\end{align}}
\def\*[#1\]{\begin{align*}#1\end{align*}}

\newcommand{\Reals}{\mathbb{R}}

\DeclareMathOperator*{\newlim}{\mathrm{lim}\vphantom{\mathrm{infsup}}}

\DeclareMathOperator*{\newmax}{\mathrm{max}\vphantom{\mathrm{infsup}}}

\renewcommand{\lim}{\newlim}

\renewcommand{\max}{\newmax}

\newcommand{\KL}[2]{\mathrm{KL}\left(#1 || #2\right)}

\newcommand{\expect}{\mathbb{E}}
\newcommand{\data}{\mathcal{D}}

\definecolor{mydarkblue}{rgb}{0,0.08,0.45}
\hypersetup{colorlinks=true,
    linkcolor=mydarkblue,
    citecolor=mydarkblue,
    filecolor=mydarkblue,
    urlcolor=mydarkblue}

\definecolor{mygreen}{HTML}{31a354}
\definecolor{myblue}{HTML}{3182bd}

\usepackage[accepted]{icml2025}

\icmltitlerunning{Reward-aware Preference Optimization Unifies Model Alignment}

\begin{document}

\twocolumn[
\icmltitle{Reward-aware Preference Optimization: A Unified Mathematical Framework for Model Alignment
}



\icmlsetsymbol{equal}{*}

\begin{icmlauthorlist}
\icmlauthor{Shengyang Sun}{nv}
\icmlauthor{Yian Zhang*}{nv}
\icmlauthor{Alexander Bukharin*}{nv}
\icmlauthor{David Mosallanezhad}{nv}
\icmlauthor{Jiaqi Zeng}{nv}
\icmlauthor{Soumye Singhal}{nv}
\icmlauthor{Gerald Shen}{nv}
\icmlauthor{Adithya Renduchintala}{nv}
\icmlauthor{Tugrul Konuk}{nv}
\icmlauthor{Yi Dong}{nv}
\icmlauthor{Zhilin Wang}{nv}
\icmlauthor{Dmitry Chichkov}{nv}
\icmlauthor{Olivier Delalleau}{nv}
\icmlauthor{Oleksii Kuchaiev}{nv}
\end{icmlauthorlist}

\icmlaffiliation{nv}{NVIDIA}
\icmlcorrespondingauthor{Shengyang Sun}{shengyangs@nvidia.com}

\icmlkeywords{Machine Learning, ICML}

\vskip 0.3in
]



\printAffiliationsAndNotice{\icmlEqualContribution} 

\begin{abstract}

The rapid development of large language model (LLM) alignment algorithms has resulted in a complex and fragmented landscape, with limited clarity on the effectiveness of different methods and their inter-connections. This paper introduces Reward-Aware Preference Optimization (RPO), a mathematical framework that unifies popular preference optimization techniques in LLM alignment, including DPO, IPO, SimPO, and REINFORCE (LOO), among others. RPO provides a structured approach to disentangle and systematically study the impact of various design choices—such as the optimization objective, the number of responses per prompt, and the use of implicit versus explicit reward models—on LLM preference optimization. We additionally propose a new experimental setup that enables the clean and direct ablation of such design choices. Through an extensive series of ablation studies within the RPO framework, we gain insights into the critical factors shaping model alignment, offering practical guidance on the most effective strategies for improving LLM alignment.

\end{abstract}

\section{Introduction}

Model alignment is the training stage that makes large language models (LLMs) helpful, honest, and harmless \citep{bai2022training, ouyang2022training}. The typical alignment pipeline consists of supervised fine-tuning (SFT), which trains the model with expert demonstrations, and preference optimization, where the model learns from human or AI feedback. Following the reinforcement learning from human feedback (RLHF) algorithm that powers ChatGPT~\citep{achiam2023gpt}, researchers have presented extensive variants of preference optimization algorithms \citep{ouyang2022training, rafailov2024direct, ahmadian2024back} to advance alignment performance further.

The plethora of preference optimization algorithms has led to a highly intricate landscape, characterized by diverse design choices and their varying degrees of effectiveness. Various offline training objectives are introduced, such as DPO \citep{rafailov2024direct}, IPO \citep{azar2024general}, and SimPO \citep{meng2024simpo}., but which one performs the best? For each objective, responses can be collected either offline \citep{rafailov2024direct} or online at the training stage \citep{guo2024direct}. 
Other factors to consider include implicit versus explicit reward models (e.g. DPO vs RLHF), the number of responses to sample per prompt, and multi-iteration alignment. Which design choices matter the most and which combination works the best?

This paper presents reward-aware preference optimization (RPO), a framework that unifies various popular preference optimization techniques in LLM alignment, including DPO, IPO, SimPO, and online RLOO, among others. The RPO framework enables us to tweak different design choices and study their impact through controlled experiments. For example, tweaking the metric function in RPO allows us to compare the effects of different training objectives such as DPO, IPO and SimPO. RPO is also so flexible that we can easily tweak the number of responses per prompt as well as switch between online and offline responses to observe their impact. After developing a better understanding of each individual design factors using RPO, we can combine the most effective elements and naturally create new algorithms such as \textit{online RPO-bwd} that outperforms existing ones.

While existing work typically compares alignment algorithms based on existing training datasets and popular academic benchmarks  \citep{ivison2024unpacking, liu2024understanding, song2024importance}, this might obscure learnings due to the indirect connection bewteen the datasets and the benchmarks, e.g., does training over HH-RLHF \citep{ouyang2022training} necessarily improve MT bench \citep{zheng2023judging}? To enable a clean evaluation of alignment algorithms, we present a synthetic experiment, where the "Ground-Truth Judge" (e.g., human annotators) of preference data is available. We can evaluate the aligned model's performance over the Ground-Truth Judge's preferences. In this setting, the performance clearly indicates the ability of each alignment algorithm to optimize the underlying judge's preferences.

Based on the experiment setup, we conduct an extensive set of ablations by varying the design choices (objectives, online vs offline, number of responses, iterative alignment, reward model quality) within the RPO framework. The results from these ablations
provide valuable learnings about preference optimization algorithms. We summarize these learnings into a cookbook for model alignment at the end of this paper.

\begin{table*}[t]
\centering
\caption{RPO recovers many existing offline preference optimization algorithms. For the loss function, we only show the loss for one single preference triplet $(x, y_1, y_2)$. We further let $\delta_{r_\pi}(x, y_1, y_2) :=  \log\frac{\pi (y_1| x)}{\pi_{ref} (y_1|x)} -  \log\frac{\pi (y_2| x)}{\pi_{ref} (y_2|x)}, \delta_{r^\star}:= \eta\left(r^\star(x, y_1) - r^\star(x, y_2) \right)$. \label{tab:rpo-recover-variants}}
\resizebox{0.95\textwidth}{!}{%
\begin{tabular}{ccc}
\toprule
Algorithm                                & Loss Function & RPO's Setting \\
\midrule
DPO                                      & $- \log \sigma\left(\beta \delta_{r_\pi}(x, y_1, y_2)\right)$        & $ \mathbb{D}^{bwd};  \delta_{r^\star} = \infty $                \\
cDPO                                     & $ - \left[ c \log \sigma\left(\beta \delta_{r_\pi}(x, y_1, y_2)\right) + (1-c)  \log \sigma\left(\beta \delta_{r_\pi}(x, y_2, y_1)\right) \right]$       & $\mathbb{D}^{bwd}; \delta_{r^\star} = \sigma^{-1}(c)      $                         \\
IPO                                   & $ \left(\delta_{r_\pi}(x, y_1, y_2) - \frac{1}{2\beta}\right)^2$          & $\mathbb{D}^{sq}; \delta_{r^\star} = \frac{1}{2}      $                         \\
Distill DPO                              &     $ \left(\beta \delta_{r_\pi}(x, y_1, y_2) -\delta_{r^\star} \right)^2  $       &            $\mathbb{D}^{sq}$              \\
SimPO                                    &       $ -\log \sigma\left( \frac{\beta}{|y_1|}  \log \pi (y_1| x) -\frac{\beta}{|y_2|}  \log \pi (y_2| x) - \gamma \right)$        &                if $\gamma=0$: $ \mathbb{D}^{bwd}, \log \pi_{ref}(y|x) \propto |y|$          \\
BRAIn &         \text{Equation~22} in \citet{pandey2024brain}     &        $ \mathbb{D}^{bwd}$                  \\
DNO & Algorithm~1 in \citet{rosset2024direct} & $r^\star(x, y) = \expect_{y'\sim \pi} \mathcal{P}(y \succ y' | x)$ \\
SteerLM 2.0      &  \text{Equation~11} in \citet{wang2024helpsteer2}  &  $r^\star(x,y)=\log\frac{P(a|y,x)P(y|x)}{Q'(y|x,a)}$ \\
\bottomrule
\end{tabular}
}
\end{table*}

\section{Background}
\paragraph{Notation.} 
We denote a prompt by $x $, a response by $y$, the policy model $\pi_{\theta}$ and the reference model $\pi_{ref}$, which is usually a supervised-fine-tuned (SFT) model and the initialization of the preference optimization stage. We denote a reward model (RM) $r$ which evaluates the quality of a response $y$ given the prompt $x$. In preference optimization, the dataset $\data = \{(x_i, y_i^{1}, y_i^2)\}_{i=1}^n$ comprises the (prompt, chosen response, rejected response) triplet unless otherwise specified. Here we assume $y_i^1$ is the chosen response and $y_i^2$ is the rejected response. It can also be extended to the setting with $K$ responses per prompt $\data = \{(x_i, y_i^{1}, ..., y_i^K)\}_{i=1}^n$.

\paragraph{Preference Optimization.} Given an SFT-ed checkpoint $\pi_{ref}$, preference optimization further improves the model by informing the model of the type of responses that are preferred by humans. Typical preference optimization algorithms include Reinforcement Learning from Human Feedback (RLHF) \citep{ouyang2022training, bai2022training} and Direct Preference Optimization \citep{rafailov2024direct}.

\paragraph{Reinforcement Learning from Human Feedback (RLHF).} RLHF first trains a reward model $r_{\phi}$ to approximate human preferences over responses. Given a preference dataset, one can train $r_{\phi}$ using the Bradley-Terry model:
\begin{align}
    \max_{\phi} \expect_{(x, y^1, y^2)} \log \frac{\exp(r_{\phi}(x, y^{1}))}{\exp(r_{\phi}(x, y^{1})) + \exp(r_{\phi}(x, y^{2}))}. \notag 
\end{align}
The next step is to optimize the policy $\pi_{\theta}$ to maximize $r_{\phi}$'s ratings of its generations subject to a KL regularization term:
\begin{align}\label{eq:rlhf-objective}
    \max_{\pi_{\theta}} \expect_{x, y\sim \pi_{\theta}} r_{\phi}(x, y) - \beta \KL{\pi_{\theta}(y|x)}{\pi_{ref}(y|x)},
\end{align}
where $\beta \ge 0$. Popular RL algorithms include Proximal Policy Optimization (PPO) \citep{schulman2017proximal}, REINFORCE  Leave-One-Out (RLOO) \citep{ahmadian2024back}, and GRPO \citep{shao2402deepseekmath}.

\paragraph{Direct Preference Optimization.} While we can use RL to maximize Eq~\ref{eq:rlhf-objective},  \citet{rafailov2024direct} shows that the optimal policy in Eq~\ref{eq:rlhf-objective} has a close-form expression:
\begin{align}
   \pi_r^\star(y|x) \propto  \pi_{ref}(y|x) \exp(r(x, y) / \beta).
\end{align}
This means a reward model can be represented by its corresponding optimal policy network: 
\begin{align}
    r_{\pi}(x, y) = \beta \log \frac{\pi (y| x)}{\pi_{ref} (y|x)}  + \beta \log Z(x). 
\end{align}
where $\log Z(x)$ is the partition function independent of the response $y$. Optimizing this implicit reward model using the Bradley-Terry model results in the differentiable objective:
\begin{align}
    \max_{\theta} \frac{1}{n}\sum_{i=1}^n \log \sigma\left(\beta  \log\frac{\pi (y_i^1| x_i)}{\pi_{ref} (y_i^1|x_i)} - \beta  \log\frac{\pi (y_i^2| x_i)}{\pi_{ref} (y_i^2|x_i)}\right). \notag 
\end{align}


\section{Reward-aware Preference Optimization Unpacks Preference Optimization Factors}\label{sec:rpo} 
This section shows that Reward-aware Preference Optimization (RPO) is a framework that unifies various alignment algorithms scuh as DPO and RLOO. We can use RPO to disentangle the impact of different design choices in alignment such as implicit vs explicit RM, online vs offline algorithms, paired vs multiple responses, etc. 

\subsection{Reward-aware Preference Optimization}


Reward-aware Preference Optimization (RPO) is first proposed in \citet{adler2024nemotron} as an upgraded version of DPO. Different from in DPO, which is based on qualitative reward signals ("which response is better"), RPO trains the implicit reward model using quantitative signals ("how much better is the preferred response") from a target explicit reward model $r^\star$. Instead of maximizing the reward difference between the preferred and rejected response, RPO minimizes the distance between the predictions of the implicit RM $r_{\pi_{\theta}}$ and those of a target RM $r^\star$ under a specific distance metric $\mathbb{D}: \mathbb{R} \times \mathbb{R} \to \mathbb{R}^{*}$:
\begin{align}\label{eq:rpo}
     & \mathbb{D} \left[ r_{\pi_{\theta}}(x_i, y_i^1) - r_{\pi_\theta}(x_i, y_i^2)  \| \eta  r^\star(x_i, y_i^1) - \eta  r^\star(x_i, y_i^2) \right] \notag \\
     &=: \mathcal{L}_{rpo}^{\mathbb{D}}( \pi_{\theta}, (x, y^1, y^2) |r^\star, \pi_{ref}, \beta, \eta),
\end{align}
where $(x, y^1, y^2)$ is a preference pair. $\eta \in \Reals^{*}$ and $\beta \in \Reals^{*}$ are hyperparameters controlling the reward scale and regularization. Similarly in RLHF, we assume the reward model $r^{*}$ has been trained to reflect human preference and optimizing the RPO objective improves the policy $\pi_{\theta}$. 

While \citet{adler2024nemotron} assumes a backward KL \textbf{distance metric}, a \textbf{paired data scheme}, and an \textbf{offline training setup}, RPO in fact allows much more flexibility in terms of these choices. We will discuss these design choices in the following paragraphs as well as how certain choice combinations recovers some existing alignment algorithms.

\subsection{Distance metric}
We consider two distance metric $\mathbb{D}$ in Equation~\ref{eq:rpo}.
\begin{itemize}
    \item \textbf{Squared Distance.} $\mathbb{D}^{sq}\left[a\|b\right] := \frac{1}{2}(a - b)^2$.
    \item \textbf{Backward Bernoulli KL divergence.} For $a \in \Reals$, we define a Bernoulli distribution $p_a$ as $p_a(x=1) = \frac{1}{1+e^{-a}}$. We then define the metric as the backward KL divergence $\mathbb{D}^{bwd}\left[a\|b\right]:=\mathrm{KL}\left[p_b \| p_a\right] $. 
\end{itemize}

\paragraph{RPO recovers preference optimization algorithms.} With proper combinations of 
$\mathbb{D}$, $r^\star$, $\pi_{ref}$, $\beta$, and $\eta$, RPO recovers many algorithms as shown in Table~\ref{tab:rpo-recover-variants}, including DPO \citep{rafailov2024direct}\footnote{See Appendix~\ref{app:rpo-eqs-dpo} for proof.}, cDPO \citep{mitchell2023note},  IPO \citep{azar2024general}, Distill DPO \citep{fisch2024robust}, BRAINn \citep{pandey2024brain}, DNO \citep{rosset2024direct}, SimPO \citep{meng2024simpo}  \footnote{RPO's corresponding $\pi_{ref}$ to SimPO is not a mathematically valid probability distribution.}, and SteerLM 2.0 \citep{wang2024helpsteer2}.

\subsection{Number of Responses Per Prompt}
Most preference optimization algorithms like DPO take a pair of responses. RPO, however, can be naturally extended to the multi-response scenario. Specifically, let $K$ be the number of responses per prompt and $y^{1:K}$ be $K$ responses for the prompt $x$. The two-response RPO aims to align the "chosen-rejected margin" between the implicit rewards and the explicit rewards using a distance metric $\mathbb{D}: \Reals \times \Reals \to \Reals$. For multiple responses, we use a corresponding distance metric $\mathbb{D}: \Reals^K \times \Reals^K \to \Reals$ over $K$-dimensional rewards and define the multi-response RPO objective as
\begin{align}
    \mathbb{D}\left[ r_{\pi_{\theta}}(x, y^1),..., r_{\pi_{\theta}}(x, y^K) \|\eta r^\star(x, y^1),..., \eta r^\star(x, y^K) \right]. \notag 
\end{align}
While the implicit RM $r_{\pi_{\theta}}$ contains the intractable log partition function $\log Z(x)$, for specific distance $\mathbb{D}$, $\log Z(x)$ will cancel out, similar to the two-response case. Then we have the multi-response RPO objective:
\begin{align}\label{eq:multi-resp-rpo}
    &\mathcal{L}_{rpo}^{\mathbb{D}}( \pi_{\theta}, (x, y^{1:K}) |\mathbb{D}, r^\star, \pi_{ref}, \beta, \eta)  \notag \\
    & \quad =  \mathbb{D}\left[ \begin{bmatrix}
           \beta \log\frac{\pi (y^1| x)}{\pi_{ref} (y^1|x)} \\
           \vdots \\
           \beta \log\frac{\pi (y^K| x)}{\pi_{ref} (y^K|x)}
         \end{bmatrix} \|\begin{bmatrix}
          \eta  r^\star(x, y^1)\\
           \vdots \\
           \eta r^\star(x, y^K)
         \end{bmatrix} \right].
\end{align}
Both $\mathbb{D}^{sq}$ and $\mathbb{D}^{bwd}$ can be extended to the multi-response case. The detailed derivations can be found in Appendix \ref{app:multi-response distance}. Particularly, the squared distance can be extended to the \textbf{squared distance with Leave-One-Out} ($\mathbb{D}^{sqloo}$):
\begin{align}\label{eq:d-sqloo}
    &\mathbb{D}^{sqloo}(a_{1:K}, b_{1:K}) = \frac{1}{2} \sum_{k=1}^K (\hat{a}_k - \hat{b}_k)^2, \\
    &\hat{a}_k = a_k - \frac{1}{K-1} \sum_{j \neq k} a_j; \hat{b}_k= b_k - \frac{1}{K-1} \sum_{j \neq k}b_j. \notag
\end{align}
The backward Bernoulli KL divergence can be extended to the \textbf{backward Categorical KL divergence}:
\begin{align}\label{eq:d-bwd-kl}
      & \mathbb{D}^{bwd}(a_{1:K}, b_{1:K})= \sum_{i=1}^K q^b_i \left(\log q^b_i - \log q^{a}_i\right), \\
      & q^b_i = \frac{\exp(b_i)}{\sum_{j=1}^K \exp(b_j)};  q^a_i = \frac{\exp(a_i)}{\sum_{j=1}^K \exp(a_j)}. \notag 
\end{align}
In both cases, the log partition function cancels out and the multi-response RPO objective becomes computable.

\subsection{Online Responses}\label{sec:online-rpo}
So far in our discussion, we assume offline RPO, where the preference dataset is collected beforehand. RPO can also be easily extended to an online algorithm. Specifically, for each training step, we sample K responses $y^{1:K}$ from the model given prompt $x$ and optimize the loss function $\mathcal{L}_{rpo}^{\mathbb{D}}( \pi_{\theta}, y^{1:K}, x |\mathbb{D}, r^\star, \pi_{ref}, \beta, \eta).$ Since it is fully differentiable, it can be directly optimized with gradient descent:
\begin{align}
    \theta \leftarrow \text{optimizer}(\theta, \nabla_{\theta} \mathcal{L}_{rpo}^{\mathbb{D}}( \pi_{\theta}, (x, y^{1:K}) |\mathbb{D}, r^\star, \pi_{ref}, \beta, \eta)). \notag 
\end{align}

\paragraph{The RPO gradient resembles a REINFORCE estimator.} 
With a distance metric $\mathbb{D}$ such that the log partition function cancels out, the RPO objective's gradient  $ \nabla_{\theta}  \mathcal{L}_{rpo}^{\mathbb{D}}$ is:
\begin{align}
    & \sum_{k=1}^K \nabla_{r^{\pi_{\theta}}(x, y^k)} \mathbb{D}\left[  r^{\pi_{\theta}}(x, y^{1:K}) \| \eta r^{\star}(x, y^{1:K}) \right] \nabla_{\theta} r^{\pi_{\theta}}(x, y^k) \notag \\
    &= \beta \sum_{k=1}^K S_k \nabla_{\theta} \log \pi_{\theta}(y^k|x). \\
    & S_k := \nabla_{r^{\pi_{\theta}}(x, y^k)} \mathbb{D}\left[ r^{\pi_{\theta}}(x, y^{1:K}) \| \eta r^{\star}(x, y^{1:K}) \right]. \notag 
\end{align}
Interestingly, the gradient resembles a REINFORCE gradient estimator  \citep{williams1992simple}, where $S_k$ is the "scale" of each score function  $\log \pi_{\theta}(y^k|x)$. 

\paragraph{RPO recovers RLOO \citep{ahmadian2024back}.} When choosing the squared distance with Leave-One-Out (rpo-sqloo), it is equivalent to the RLOO algorithm, as shown below. Based on formula of $\mathbb{D}^{sqloo}$ in Eq~\ref{eq:d-sqloo}, the gradient
    \begin{align}
        \nabla_{a_k} \mathbb{D}^{sqloo}(a_{1:K}, b_{1:K}) 
        = \frac{K}{K-1} (\hat{a}_k - \hat{b}_k). \notag 
    \end{align}
We define the reinforce gradient as the following equation, which equals the score function's scale when adopting the vanilla REINFORCE gradient estimator \citep{williams1992simple}.
\begin{align}
r_{reinforce}^{\star, \pi_{\theta}}(x, y^k; \beta, \eta):= \beta \log \frac{\pi_{\theta}(y^k|x)}{\pi_{ref}(y^k|x)} - \eta r^{\star}(x, y^k) . \notag 
\end{align}
Then we compute the scale $S_k$,
\begin{align}
   & \frac{K-1}{K} S_k = \left(r^{\pi_{\theta}}(x, y^k) - \frac{1}{K-1}\sum_{j\neq k} r^{\pi_{\theta}}(x, y^j)\right)  \notag  \\
     & \quad \quad \quad  -\eta \left(r^{\star}(x, y^k) - \frac{1}{K-1}\sum_{j\neq k} r^{\star}(x, y^j)\right)  \notag \\
    & = r_{reinforce}^{\star, \pi_{\theta}}(x, y^k; \beta, \eta) - \frac{1}{K-1}\sum_{j\neq k} r_{reinforce}^{\star, \pi_{\theta}}(x, y^j; \beta, \eta). \notag 
\end{align}
The above shows RPO recovers the REINFORCE Leave-One-Out when $\eta=1$ \citep{ahmadian2024back}. 

\paragraph{RPO with backward Categorical KL (rpo-bwd).} When adopting the backward KL $\mathbb{D}^{bwd}$ in Eq~\ref{eq:d-bwd-kl}, we can compute the RPO objective's gradient as follows.
\begin{align}
    S_k   &= q_k^{ r^{\pi_{\theta}}(x, y^{1:K})} - q_k^{\eta r^\star(x, y^{1:K})}.
\end{align}
The first term are the softmax probabilities with $\{\beta \log \frac{\pi_{\theta}(y^k|x)}{\pi_{ref}(y^k|x)} \}_{k=1}^K$ as logits; the second term are the softmax probabilities with $\{\eta r^\star(x, y^k)\}_{k=1}^K$ as logits. Thus online RPO-bwd is equivalent to replacing the "scale" in the REINFORCE estimator as the difference between the softmax probabilities of ground-truth rewards and predicted rewards. The derivation can be found in Appendix \ref{app:gradient-rpo-bwd}. Computing the softmax has a normalization effect over rewards, similar to the variance normalization in GRPO \citep{shao2402deepseekmath}.

\section{The Experiment Setup}\label{sec:exp-design}
The DAG below shows key elements in model alignment.
\begin{align}
    \textit{Ideal Goal} \rightarrow \textit{Formalized Goal} \overset{\text{Alignment}}{\longrightarrow} \textit{Models}  \rightarrow  \textit{Evals} \notag 
\end{align}
Model alignment cares about the \textit{Ideal Goal}, that is to train models that are helpful, honest, and harmless \citep{ouyang2022training, bai2022training}. While the \textit{Ideal Goal} is ambiguous, a \textit{Formalized Goal} is used to approximate it, such as human annotator preferences. Then we run model alignment algorithms like PPO and DPO to align models towards the \textit{Formalized Goal}. Finally, we evaluate how well does the model perform regarding the \textit{Ideal Goal} using a wide range of benchmarks such as MMLU \citep{hendrycks2020measuring}, HumanEval~\citep{chen2021evaluating}, and Lmsys Chatbot Arena \citep{chiangchatbot}. While the alignment process involves many factors, such as human annotators, alignment algorithms, and benchmarks, this work focuses on the impact of alignment algorithms, i.e., from \textit{Formalized Goal} to \textit{Models}.

Based on the discussions above, we propose the following experiment setup to ablate alignment algorithms in a clean and direct way. The essence of the design is to compare which algorithm optimizes the \textit{Formalized Goal} the best.


\subsection{Experiment Setup}


\textbf{Ground-Truth Judge.} To contextualize the \textit{Formalized Goal}, we first choose the "Ground-Truth" judge. The whole purpose of alignment algorithms is to optimize the "Ground-Truth" judge's preferences over its generations. Since human labelers are costly, we use a reward model as the synthetic Ground-Truth judge. 
Specifically, we choose {Nemotron-4-340B-RM} \citep{wang2024helpsteer2}, one of the leading reward models on the \textit{Reward Bench} leaderboard \citep{lambert2024rewardbench}. In fact, any reasonable RM can be used in our setup since it is assumed to be the Ground-Truth.

\textbf{Model Initialization.} Since this experiment focuses on the preference optimization stage, we initialize the model from a Supervised Fine-Tuning (SFT) checkpoint, which is created by training the \textit{llama3-8b/70b-base} model \citep{dubey2024llama} on 100k responses generated by {Nemotron-4-340B-Instruct} \citep{adler2024nemotron}, whose prompts come from the \textit{lmsys-1M} dataset \citep{zheng2023lmsyschat1m}.

\textbf{Preference Training Datasets.} We use 120k \textit{lmsys-1M} prompts to build the preference dataset, which are kept disjoint from the ones used in building the SFT dataset. Given each prompt, multiple responses are sampled from the SFT model and annotated using the Ground-Truth RM (i.e., Nemotron-4-340B-RM).  We then select the highest-reward response as the chosen response and a random response as the rejected response to create a preference triplet.

\textbf{Evaluation.} As we focus on optimizing the Ground-Truth Judge's preference, we naturally choose its predicted reward as the evaluation metric. We consider three sets of evaluation prompts: 1024 \textit{lmsys (valid)} prompts for hyperparameter tuning and checkpoint selection, 1024 \textit{lmsys (test)} prompts for in-distribution evaluation, and 805 \textit{alpacaeval} prompts \citep{dubois2024length} for out-of-distribution evaluation. For each prompt set, we generate responses from the trained model, annotate rewards using the RM, and then calculate the final metrics: \textit{average reward} and \textit{average win-rate over the SFT checkpoint} as well as their 95\% confidence interval.

\begin{table*}[t]
\centering
\caption{The average reward and win-rate of model alignment algorithms.. We compute the metrics over \textit{lmsys (test)} prompts and \textit{alpacaeval} prompts. We train models starting with \textit{llama3-8b-sft} and \textit{llama3-70b-sft}; with different objectives (\textit{dpo, simpo, kto, rpo-bwd, rpo-sqloo}); with different numbers of responses per prompt (K=2, K=4); with {\color{myblue}{offline}} or {\color{mygreen}{online}} responses; with the Ground-Truth RM or a Learnt RM from the preference dataset. For 8b, we report 95\% confidence intervals over three runs with the best hyper-parameter. For 70b, we run once for the best hyper-parameter due to compute limitations. ${\color{red}{^\dagger}}$\textit{Online rpo-sqloo} is equivalent to \textit{REINFORCE Leave-One-Out}. \label{tab:offline-online}}
\resizebox{0.98\textwidth}{!}{%
\begin{tabular}{ccccccccc}
\toprule
             &   &  & & & \multicolumn{2}{c}{AvgReward}                              & \multicolumn{2}{c}{Win-Rate ($\%$)}                                  \\
             Base &    Loss & K & Responses & RM  & lmsys (test)      & alpacaeval           & lmsys (test)        & alpacaeval         \\
\midrule
\multirow{10}{*}{8b} & sft &  & &   & 5.284 & 5.383 & 50 & 50 \\
 & dpo & 2  & {\color{myblue}{Offline}}  & GT & $5.503 \pm 0.016$ & $5.600 \pm 0.006$ & $70.5 \pm 1.2$  & $71.6 \pm 2.0$ \\
& simpo & 2  & {\color{myblue}{Offline}} & GT  & $5.533 \pm 0.006$ & $5.646 \pm 0.016$  & $71.6 \pm 1.1$ & $72.0 \pm 1.6$ \\
& kto & 2  & {\color{myblue}{Offline}} & GT  & $5.453 \pm 0.017$ & $5.509 \pm 0.012$ & $65.8 \pm 3.0$  & $62.8 \pm 1.7$ \\
& rpo-bwd & 2  & {\color{myblue}{Offline}}  &  GT      & $5.496\pm 0.014$ & $5.623 \pm 0.016$ & $69.1\pm 1.5$ & $72.8\pm 0.8$ \\
 & rpo-bwd & 4         & {\color{myblue}{Offline}} & GT    & $5.525\pm 0.007$ & $5.618\pm 0.002$  & $70.2\pm 1.5$ & $72.3\pm 0.3$ \\
 & rpo-sqloo & 2  & {\color{myblue}{Offline}} & GT  & $5.448 \pm 0.011$ & $5.556 \pm 0.007$ & $63.3\pm 1.3$  & $62.8 \pm 0.2$  \\
 & rpo-sqloo & 4  & {\color{myblue}{Offline}} & GT  & $5.445 \pm 0.008$ & $5.573 \pm 0.013$ &  $60.1 \pm 0.3$   & $63.8 \pm 1.1$ \\
& rpo-bwd & 4  & {\color{mygreen}{Online}} & GT & $\mathbf{5.66 \pm 0.03}$ & $\mathbf{5.685 \pm 0.035}$  & $\mathbf{78.6 \pm 0.1}$ & $\mathbf{77.9 \pm 1.0}$  \\
 & rpo-sqloo${\color{red}{^\dagger}}$ & 4  & {\color{mygreen}{Online}} & GT & $5.617 \pm 0.008$ & $5.672 \pm 0.047$  &$77.9 \pm 0.5$ & $74.8 \pm 1.1$  \\
 & rpo-bwd & 4 & {\color{mygreen}{Online}}  & Learnt & $5.355 \pm 0.014$ & $5.472 \pm 0.013$ & $57.3 \pm 1.6$ & $58.5 \pm 1.9$ \\
 & rpo-sqloo${\color{red}{^\dagger}}$ & 4  & {\color{mygreen}{Online}}  & Learnt  & $5.335 \pm 0.002$ & $5.442 \pm 0.013$ & $57.3 \pm 1.2$ & $56.7 \pm 0.9$ \\
\midrule
\multirow{10}{*}{70b} & sft &  & &   & 5.469 & 5.671 & 50 & 50 \\
 & dpo & 2  & {\color{myblue}{Offline}}  & GT & 5.805 & 5.926  & 81.7 & 82.2 \\
& simpo & 2  & {\color{myblue}{Offline}} & GT  & 5.794 & 5.907 & 79.6 &  77.0 \\
& kto & 2  & {\color{myblue}{Offline}} & GT  & 5.657 & 5.799  &
68.9 &
68.2\\ 		
& rpo-bwd & 2 & {\color{myblue}{Offline}} & GT  & 5.774 & 5.910 & 78.7 & 80.9 \\
 & rpo-bwd & 4  & {\color{myblue}{Offline}} & GT  & 5.788 & 5.921 & 77.1 & 81.0   \\
 & rpo-sqloo & 2 & {\color{myblue}{Offline}}  & GT  & 5.710 & 5.863 & 72.1 & 73.7  \\
 & rpo-sqloo & 4 & {\color{myblue}{Offline}} & GT  & 5.740 &  5.878   & 76.1 & 73.4   \\
 & rpo-bwd & 4 & {\color{mygreen}{Online}}  & GT  & $\mathbf{5.916}$ & $\mathbf{5.992}$ &  $\mathbf{85.7}$  & $\mathbf{85.5}$  \\
 & rpo-sqloo${\color{red}{^\dagger}}$ & 4  & {\color{mygreen}{Online}}  & GT & 5.796 & 5.924 &  78.5  & 76.1  \\
 & rpo-bwd & 4 & {\color{mygreen}{Online}}  & Learnt  & 5.503 & 5.695 & 52.9 & 52.2  \\
 & rpo-sqloo${\color{red}{^\dagger}}$ & 4  & {\color{mygreen}{Online}}  & Learnt  & 5.484 & 5.685  & 52.6 & 53.8 \\
\bottomrule
\end{tabular}
}
\end{table*}

\subsection{The impact of prompt distributions}\label{subsec:prompt-impacts}

\begin{figure}[h]
    \centering
    \includegraphics[width=\linewidth]{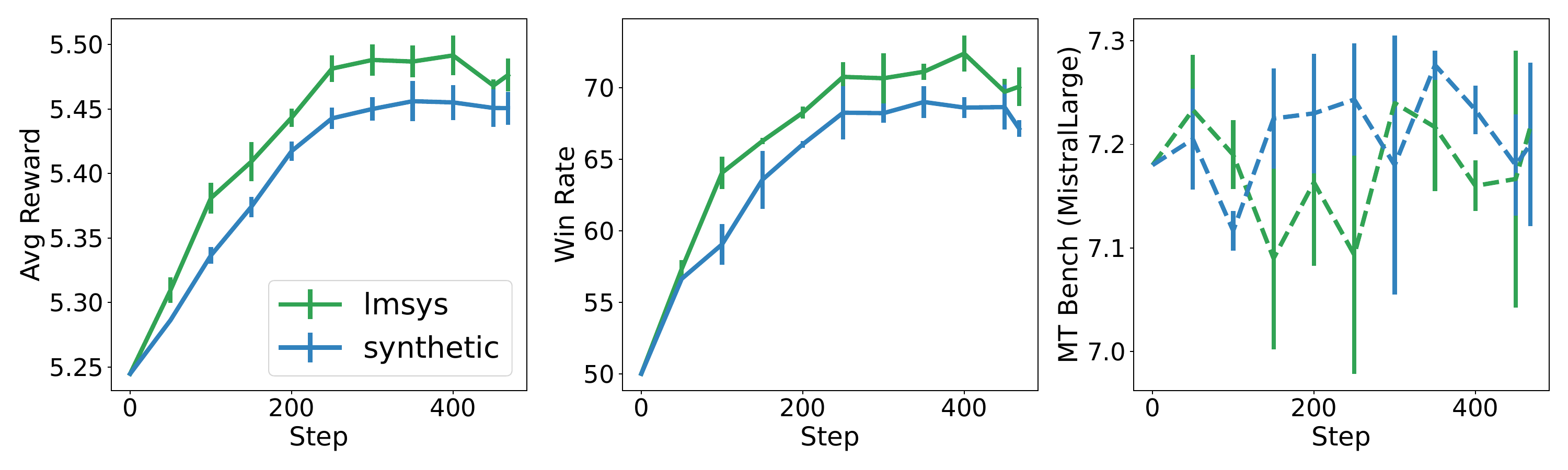}
    \caption{The average reward (\textit{left}) and win-rate (\textit{mid}) over \textit{lmsys (valid)} prompts along training. The \textit{right} figure shows the MT bench (judged by Mistral Large 2). Error bars represent 95\% confidence intervals over 3 independent runs. We compare two training datasets, which are generated by the llama3-8b-sft model using \textit{{\color{mygreen}lmsys}} and \textit{{\color{myblue} synthetic}} prompts, respectively. We observe training on in-distribution  \textit{lmsys} prompts achieves higher rewards than training on out-of-distribution  \textit{synthetic} prompts. However, the MT-Bench metric has a large variance, hardly showing any learnings.}
    \label{fig:prompts}
\end{figure}


We conduct an initial experiment to show that the proposed experiment framework is a better tested to study alignment algorithms. Specifically, we aim to study the impact of training on \textit{in-distribution} prompts and \textit{out-of-distribution} prompts. For the \textit{in-distribution} prompts, we generate a preference training dataset using the \textit{llama3-8b-sft} model over \textit{lmsys} prompts. For the \textit{out-of-distribution} prompts, we generate another preference dataset using the same model over \textit{synthetic} prompts, which were generated with the similar pipeline as UltraChat \citep{ding2023enhancing} and Daring-Anteater \citep{wang2024helpsteer2}. In Figure~\ref{fig:prompts}, we plot how the metrics evolve along with training steps. Specifically, we compare \textit{validation avg reward}, \textit{validation win-rate}, and the \textit{MT Bench} benchmark \citep{zheng2023judging}, which is a popular benchmark for evaluating aligned models. 

Based on Figure~\ref{fig:prompts}, we observe that both the proposed \textit{avg reward} and \textit{win-rate} have a clear increase as training progresses, while the MT bench metric fluctuates a lot. The contrast demonstrates that \textbf{our proposed experiment setup can better investigate the effect of alignment algorithms, while existing benchmarks might obscure the learnings}. We also observe that training on in-distribution \textit{lmsys} prompts achieves higher average reward and win-rate than training on out-of-distribution  \textit{synthetic} prompts, which matches our intuition. Besides the validation performances in Figure~\ref{fig:prompts}, we further report the performances over the \textit{lmsys(test)} and \textit{alpacaeval} prompts in Table~\ref{tab:in-out-dist-prompts} in the appendix.


\section{Related Works}

\textbf{Preference optimization algorithms.} The pioneering RLHF works \citep{ouyang2022training, bai2022training} adopted Proximal Policy Optimization (PPO) \citep{schulman2017proximal} to optimize the RLHF objective in Eq~\ref{eq:rlhf-objective}. To improve memory efficiency and optimization stability, \citet{ahmadian2024back} and \citet{shao2402deepseekmath} forgo the critic in PPO with the REINFORCE Leave-One-Out (RLOO) and the Group Relative Policy Optimization (GRPO), respectively. RLOO and GRPO have separately powered the alignment of LLama3-Nemotron-70b-Instruct~\citep{wang2024helpsteer2preference} and DeepSeek-R1~\citep{guo2025deepseek}. In contrast to online RL methods, \citet{rafailov2024direct} proposed DPO which directly optimizes the policy network using an offline dataset according to its connection to the implicit RM. Researchers then proposed various objectives beyond DPO, including noisy DPO \citep{mitchell2023note}, IPO \citep{azar2024general}, SimPO \citep{meng2024simpo}, KTO \citep{ethayarajh2024kto}, distill DPO \citep{fisch2024robust}, DNO \citep{rosset2024direct}, BRAIn~\citep{pandey2024brain}, APO~\citep{d2024anchored}, WPO~\citep{zhou2024wpo}, and SteerLM 2.0~\citep{wang2024helpsteer2}. SPIN \citep{chenself} adopted the DPO objective in supervised fine-tuning and showed consistent improvement with iterative alignment. \citet{adler2024nemotron} introduced RPO in their Nemotron-4-340b-Instruct training focusing on backward KL, paired responses, and offline training. In this work, we extend RPO to different design choices and demonstrate it as a roadmap to connect various online and offline preference optimization algorithms. 

\textbf{Studies of model alignment algorithms.} Given the complicated landscape of model alignment algorithms, researchers conducted various studies to investigate their performances. Most approaches \citep{ivison2024unpacking, liu2024understanding, song2024importance, wang2024helpsteer2} relied on training over existing datasets (e.g., HH-RLHF~\citep{ouyang2022training}) and evaluating over existing benchmarks (e.g., MT bench~\citep{zheng2023judging}). However, the indirect relationship between the dataset and the benchmark might generate misleading conclusions, like we showed in Figure~\ref{fig:prompts}. \citet{tang2024understanding} conducted carefully designed ablations to investigate the performance gap between online and offline algorithms. \citet{lin2024limited} showed that implicit RMs usually generalizes worse than explicit RMs over out-of-distribution datasets. Similarly to RPO, UNA~\citep{wang2024unifying} proposed to approximate the explicit RM with the implicit RM but failed to disentangle different design factors. \citet{wang2024comprehensive} is a survey of various alignment algorithms.

\section{Ablation Results}
Section~\ref{sec:rpo} shows that RPO unifies many preference optimizing algorithms with different design choices. Section~\ref{sec:exp-design} proposes the experimental setup to ablate the performance of alignment algorithms. In this section, we vary these design choices, including objectives, number of responses, online or offline responses, and reward model qualities. We study how do these design choices affect model alignment qualities according to the proposed experimental framework.

Specifically, we train models starting with both the \textit{llama3-8b-sft} and \textit{llama3-70b-sft} models to optimize the Ground-Truth Judge (i.e., Nemotron-4-340B-RM)'s preferences. Offline methods, including DPO, SimPO, KTO, RPO-bwd, and RPO-sqloo, are trained using the preference dataset annotated by the Ground-Truth Judge. Online methods, including online RPO-sqloo (i.e., RLOO) and online RPO-bwd, optimizes the Ground-Truth Judge's preferences. Understanding the best approach to train a reward model from preference data is out of scope of this paper. We evaluate model performances according to the average rewards and win-rates over \textit{lmsys (test)} prompts (in-distribution) and \textit{alpacaeval} prompts (out-of-distribution). The full results are shown in Table~\ref{tab:offline-online} and we detail our learnings in the below.

\textbf{Objectives.} Focusing on the offline case and setting $K=2$, we compare different objectives' performances. We observe DPO, RPO-bwd, and SimPO perform better for both the 8b and the 70b model. In contrast, KTO and RPO-sqloo have consistently worse rewards and win-rates.

\textbf{Number of Responses}\label{sec:multi-rep-rep}
RPO enables to use >2 responses (K) for training. For both distance metrics (sqloo and bwd), we increase K from 2 to 4 and keep everything else the same (model initialization, prompts, evaluation) to repeat the experiments. Comparing numbers for both 8b and 70b in Table \ref{tab:offline-online}, we observe no significant improvements brought by the increase in K. This suggests that the number of responses does not affect alignment performances much in our setup.

\textbf{Comparing Online Algorithms}
 We illustrated the equivalence between online RPO-sqloo and RLOO in Section~\ref{sec:online-rpo} and proposed online RPO-bwd. In this section, we compare the performance of {\color{mygreen} online RPO-bwd} and {\color{myblue} online RPO-sqloo} (i.e. RLOO). As shown in Table~\ref{tab:offline-online}, with the same Ground-Truth RMs, {\color{mygreen} online RPO-bwd} improves significantly over {\color{myblue} online RPO-sqloo} for both 8b and 70b models. For example, the \textit{lmsys(test)} avgReward increases from 5.796 to 5.916 and its win-rate increases from $78.5\%$ to $85.7\%$ in the 70b setting. We also plot the learning curves of the \textit{lmsys(valid)} avgReward and $\KL{\pi_{\theta}}{\pi_{ref}}$ in Figure~\ref{fig:online-rpo-bwd-vs-sqloo}. As shown in the figure, {\color{mygreen} online RPO-bwd} has a slower KL divergence increase, together with faster reward increases. {\color{mygreen} Online RPO-bwd}'s training is also more stable given that {\color{myblue} online RPO-sqloo} crashes in the middle. Since the RLHF objective aims to improve the reward while regularizing the KL, this highlights that \textit{online RPO-bwd is a better optimizer for the RLHF  objective (Eq~\ref{eq:rlhf-objective})  than RLOO}.

\begin{figure}[h]
    \centering
    \includegraphics[width=\linewidth]{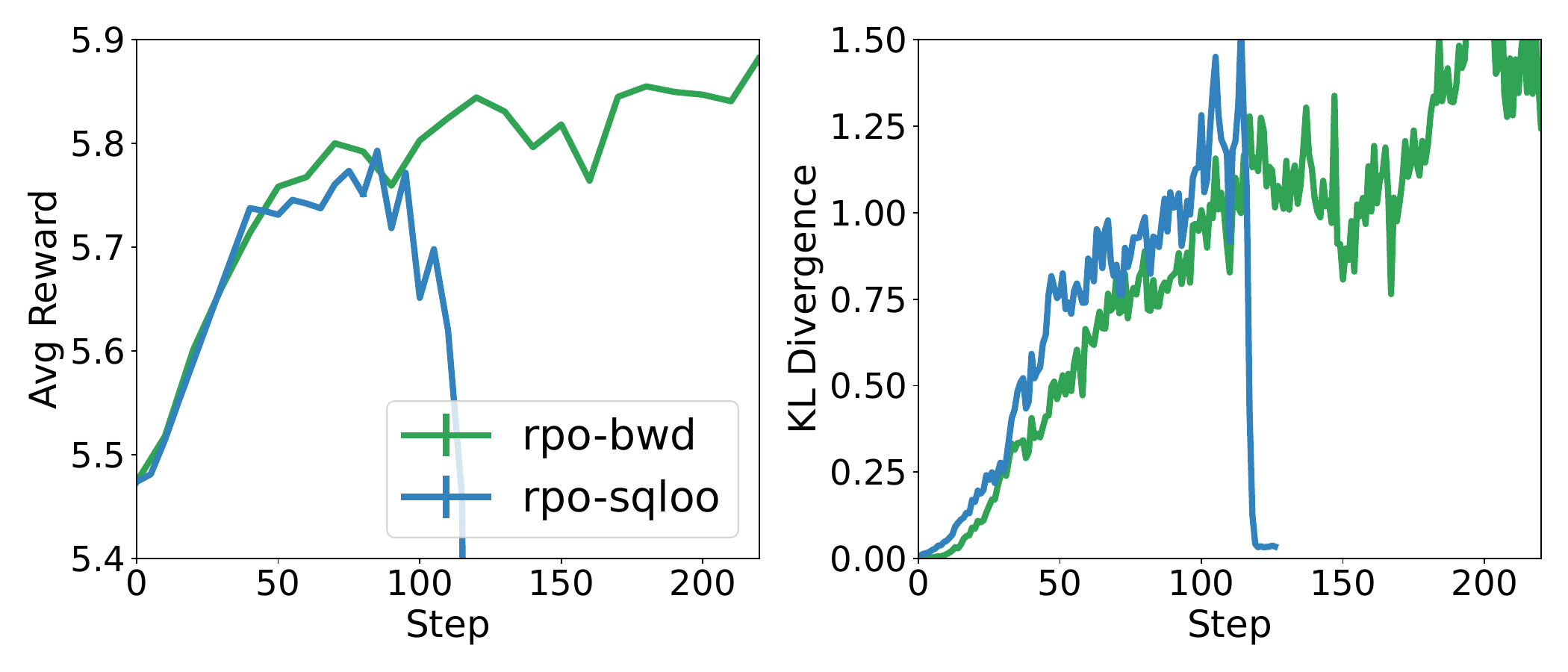}
    \caption{{\color{mygreen} Online RPO-bwd} vs {\color{myblue} online RPO-sqloo} (RLOO). We plot average rewards on \textit{lmsys(valid)} (\textit{left}) and the KL divergence with the reference policy (\textit{right}). The valid reward increases faster and the KL divergence increases slower for {\color{mygreen} RPO-bwd}. This indicates that {\color{mygreen} online RPO-bwd} can better optimize the RLHF objective (Eq~\ref{eq:rlhf-objective}) than {\color{myblue} RLOO}. In addition, {\color{myblue} RLOO}'s training exploded in the middle; while {\color{mygreen} RPO-bwd}'s training kept stable in all our runs.}
    \label{fig:online-rpo-bwd-vs-sqloo}
\end{figure}

\textbf{Online vs Offline Training}
For both \textit{rpo-bwd (K=4)} and \textit{rpo-sqloo (K=4)} and for both \textit{8b} and \textit{70b} models, we run online and offline training and show the results in Table~\ref{tab:offline-online}. Specifically, offline training adopts the preference datasets annotated by the Ground-Truth Judge; online training optimizes the policy to maximize the Ground-Truth Judge's preferences. Across all settings (rpo-bwd / rpo-sqloo; 8b / 70b), online  training outperforms offline training. For example, The 70b online rpo-bwd model has an out-of-distribution (alpacaeval) win rate of 85.5\%, while the offline counterpart gets 81.0\%. This shows an advantage of online training. It is worthy to note that, we assumed access to the Ground-Truth judge in online training. In practice, if only the human annotated preference datasets are available, online training is not necessary better than offline training since it requires to train a new reward model. The quality of the learnt reward model is critical to online method's success, as we show below in Figure~\ref{fig:training-with-learnt-rm}.

\textbf{Iterative (Online / Offline) Alignment}\label{sec:iterative}
So far we have only explored single-iteration training, i.e. we train the model until convergence and then evalute it. In this section, we study the impact of repeating this process multiple times. Specifically, let $\pi_{\theta_k}$ ($\pi_{\theta_0} = \pi_{SFT}$) be the policy after the $k^{th}$ iteration. In the $(k+1)^{th}$ iteration, we use $\pi_{\theta_k}$ as the reference policy and the initialization and continue to train the model. For offline alignment, we need to regenerate the training data based on $\pi_{\theta_k}$ and the Grund-Truth Judge. For online training, since the data is sampled along training, we only need to update the reference model at every iteration.

\begin{figure}[h]
    \centering
    \includegraphics[width=\linewidth]{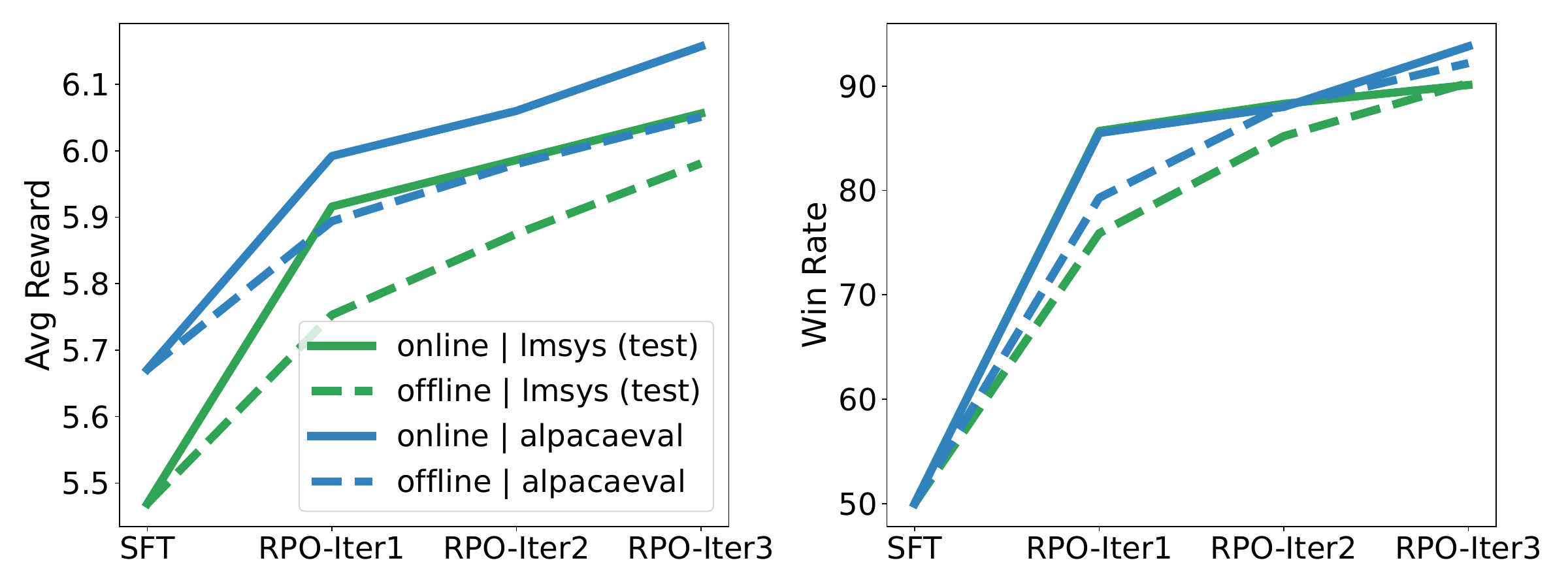}
    \caption{Performance improves consistently with more iterations.}
    \label{fig:iterative-dpo}
\end{figure}
 
 We demonstrate the effect of iterative online/offline RPO-bwd using the 70b model and show their performance in Figure \ref{fig:iterative-dpo} (more details in Table \ref{tab:iterative-methods}). As shown in the figure, as 
 the number of iterations increases, the model's performance consistently improves for both online and offline training. The strongest model (online, iter3) has an average OOD reward of 6.157 and win rate of 93.8\%, beating all rows in Table \ref{tab:offline-online}. This demonstrates the effectiveness of iterative alignment. We also observe that at each iteration, online training generally outperforms offline training, but the gap is shrinking as number of iterations increases. In Appendix~\ref{app:sec:offline-online-iterative}, we have detailed the pseudocode for online/offline RPO and iterative online/offline RPO for reference.

\textbf{The Importance of RM Quality.} So far we assumed access to the Ground-Truth Judge (the Nemotron-4-340B-RM) and applied online methods to maximize its rewards. In practice, e.g., when the Ground-Truth Judge are human annotators, we need to first learn a RM to minic human annotators' preferences and then optimize the learnt RM's preferences. To study this scenario, we assume access to the preference data only and repeat the online training experiments using a learnt RM instead of the Nemotron-4-340B-RM. Then we evaluate how does the Nemotron-4-340B-RM's predictive rewards improve. In Table~\ref{tab:offline-online}, we find that online methods (rpo-bwd, rpo-sqloo) with the Learnt RM can barely improve over the SFT model. In Figure~\ref{fig:training-with-learnt-rm}, we observe the Learnt RM's rewards keep increasing when the Ground-Truth RM's rewards drop, highlighting the \textit{reward hacking} phenomenon. The results show that a strong reward model is critical to online methods and naively training a RM on the preference dataset is likely not sufficient. We left the investigation of RM training methods to future works.

\begin{figure}[h]
    \centering
    \includegraphics[width=0.9\linewidth]{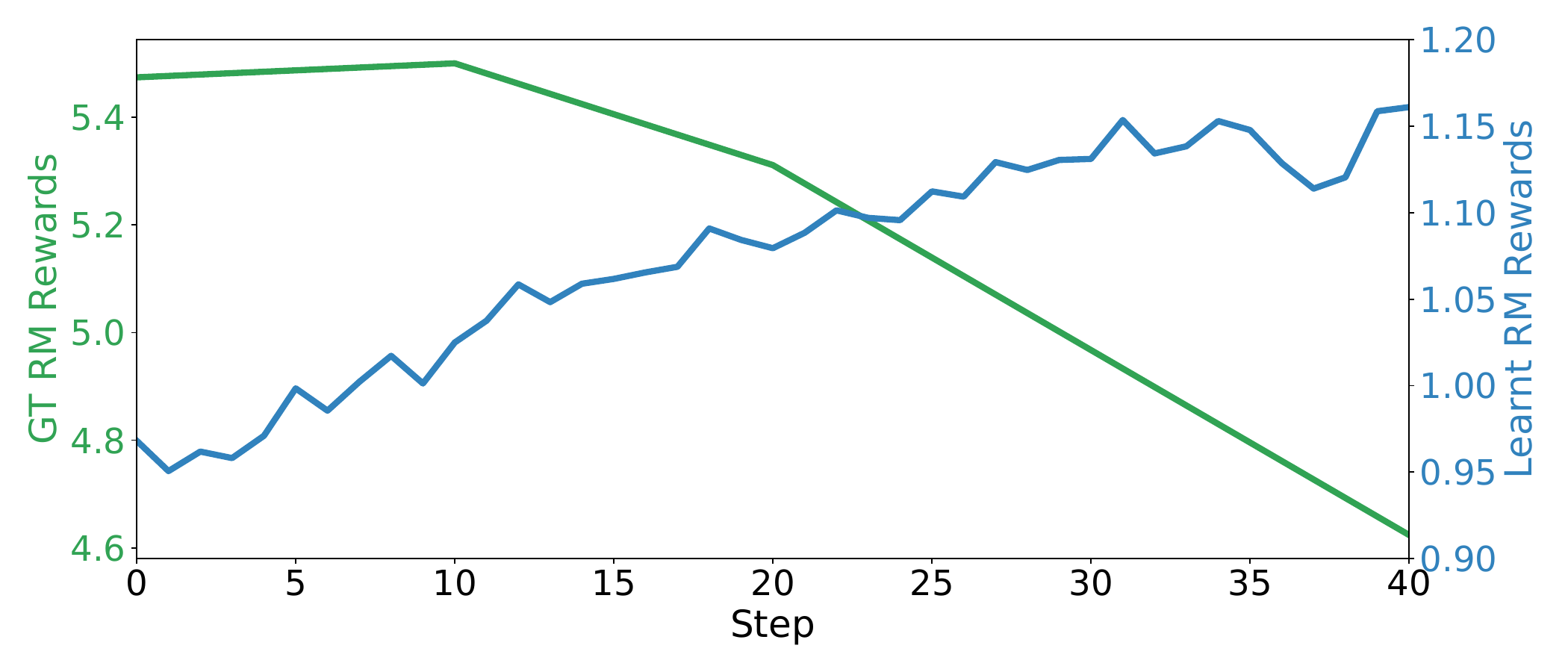}
    \caption{{\color{mygreen}GT RM}'s rewards vs {\color{myblue}Learnt RM}'s rewards along training.}
    \label{fig:training-with-learnt-rm}
\end{figure}

\section{Conclusion}

\paragraph{Summary of Learnings.} This paper has shown many learnings to understand model alignment algorithms better.
\begin{itemize}
    \item RPO unifies a variety of alignment algorithms like DPO, IPO, RLOO through different design choices. Using the framework, we can come up with competitive new methods like online RPO-bwd.
    \item The proposed experimental setup is a better testbed to ablate the effect of alignment algorithms, than popular benchmarks like MT-bench.
    \item Offline DPO, SimPO, and RPO-bwd perform similarly when well tuned, better than KTO and RPO-sqloo.
    \item The number of responses in offline preference optimization does not significantly impact model performances.
    \item Online RPO-bwd improves over online RPO-sqloo (i.e., RLOO) with more stable training, better KL regularization, and higher rewards. 
    \item Iterative alignment for both offline and online methods brings consistent performance improvement.
    \item When the preference dataset is annotated by an available reward model or verifier like Nemotron-4-340B-Instruct~\citet{adler2024nemotron} and DeepSeek-R1~\citep{guo2025deepseek}, online methods significantly outperform offline methods; when only the preference dataset is available (if it is human annotated or if the annotating RM is not available), the quality of online methods depends critically to the learnt RM. Offline methods can be competitive in this case.
\end{itemize}

\paragraph{Alignment Recipe Recommendations.} The best alignment strategy can vary in difference scenarios depending on the available compute resources, human labeling resources, and etc. We make the following two general recommendations based on the main sources of training data.
\begin{itemize}
    \item If a strong RM or a ground-truth verifier is available, we suggest running iterative online alignment to optimize the RM's preferences. This reflects NVIDIA's Llama3.1-Nemotron-70b-Instruct alignment recipe \citep{wang2024helpsteer2preference} and DeepSeek-R1's recipe~\citep{guo2025deepseek}. We recommend online \textit{RPO-bwd}.
    \item If enough human labeling resources are available, we suggest running iterative offline alignment while new preference datasets are built in each iteration. This reflects Meta's Llama3 alignment recipe \citep{dubey2024llama}. We recommend offline \textit{RPO-bwd}.
\end{itemize}

\paragraph{Limitations and Future Works.} This work opens up a new venue to study various factors in model alignment. It asks for more future works to gain deeper understanding of model alignment techniques.
\begin{itemize}
    \item \textbf{Token-Level RPO.} While RPO unifies many alignment methods, it fails to connect to PPO \citep{schulman2017proximal} in RLHF. The reason is that RPO applies on the sequence-level, lacking the ability of token-level credit assignment like PPO. We expect future works to extend the sequence-level RPO to token-level RPO.
    \item \textbf{Reward Model training.} Figure~\ref{fig:training-with-learnt-rm} showed that the Learnt RM's quality is critical to online alignment. We expect future works to understand the best recipe to train strong RMs \citep{gao2023scaling, rafailov2024scaling}.
    \item \textbf{Preference data generation.} In this paper we select random prompts, generate random responses using the current policy, and build preferences with the Ground-Truth Judge. The proposed experimental setup enables us to understand the best recipes for preference data generation including prompt selection, response generation (diversity, distillation, on-policy vs off-policy), and preference judging (RM vs LLM-as-Judge).
    \item \textbf{Extended Experimental Setups.} To understand alignment recipes in wider scenarios, we hope to see results in a wider variety of experimental setups regarding problem domains (e.g., math and coding), Ground-Truth Judges (RMs, Ground-Truth verifiers), and etc.
\end{itemize}

\clearpage

\section*{Impact Statement}

This paper advances our understanding of preference optimization algorithms in the context of model alignment. Its findings have significant implications, such as improving the development of AI systems that are more beneficial to society through enhanced algorithms. Additionally, by aligning models more closely with human values, the research contributes to improving AI safety. While the paper primarily focuses on algorithmic advancements, it also highlights the importance of adopting these algorithms with caution to ensure that strong AI models are trained in ways that mitigate the risk of misalignment.

\nocite{langley00}

\bibliography{bib/misc}
\bibliographystyle{icml2025}

\newpage
\appendix
\onecolumn

\section{Derivation of the distance functions for the multi-response scenario}\label{app:multi-response distance}

\paragraph{Squared Distance with Leave-One-Out (sqloo).} Computing the squared distance naively between the implicit rewards and explicit rewards cannot remove the log partition function. We consider the squared distance with Leave-One-Out. Specifically,
\begin{align}
    &\mathbb{D}^{sqloo}(a_{1:K}, b_{1:K}) = \frac{1}{2} \sum_{k=1}^K (\hat{a}_k - \hat{b}_k)^2, \notag  \\
    &\hat{a}_k = a_k - \frac{1}{K-1} \sum_{j \neq k} a_j; \hat{b}_k= b_k - \frac{1}{K-1} \sum_{j \neq k}b_j. \notag
\end{align}
Because of the Leave-One-Out subtraction, the log partition function $\log Z(x)$ cancels out.

\paragraph{Backward Categorical KL Divergence (bwd-kl).} We use the rewards as softmax logits to define a categorical distribution and compute the KL divergence between two categorical distributions.
\begin{align}
      & \mathbb{D}^{bwd}(a_{1:K}, b_{1:K})= \sum_{i=1}^K q^b_i \left(\log q^b_i - \log q^{a}_i\right),  \notag \\
      & q^b_i = \frac{\exp(b_i)}{\sum_{j=1}^K \exp(b_j)};  q^a_i = \frac{\exp(a_i)}{\sum_{j=1}^K \exp(a_j)}. \notag 
\end{align}
Because of the softmax operation, the log partition function $\log Z(x)$ cancels out as well.

\paragraph{Forward Categorical KL Divergence (fwd-kl).} Similarly, we can use the forward KL divergence as the metric.
\begin{align}\label{eq:d-fwd-kl}
      \mathbb{D}^{fwd}(a_{1:K}, b_{1:K})= \sum_{i=1}^K q^b_i \left(\log q^b_i - \log q^{a}_i\right).
\end{align}

\section{Gradient derivation of online RPO with the Backward KL Divergence (rpo-bwd) distance.} \label{app:gradient-rpo-bwd}
When using the backward KL divergence,
\begin{align}
  &  \nabla_{a_i} \mathbb{D}^{bwd}(a_{1:K}, b_{1:K}) 
    = \nabla_{a_i} \left[ \log\left(\sum_{j=1}^K \exp(a_i)\right) -   q^b_i a_i  \right] \notag \\
    & \quad \quad \quad \quad = \frac{\exp(a_i)}{\sum_{j=1}^K \exp(a_i)} -  q^b_i =  q^a_i - q^b_i.
\end{align}
Therefore, the score function scale is
\begin{align}
    S_k   &= q_k^{ r^{\pi_{\theta}}(x, y^{1:K})} - q_k^{\eta r^\star(x, y^{1:K})}.
\end{align}
In $S_k$, the first term is the softmax probabilities using $\beta \log \frac{\pi_{\theta}(y^k|x)}{\pi_{ref}(y^k|x)}, k=1,...,K$ as the logits; the second term is the softmax probabilities using $\eta r^\star(x, y^k), k=1,...,K$ as the logits. Therefore, using the backward KL divergence in online RPO, is equivalent to replacing the "scale" in the REINFORCE gradient estimator as the difference between the softmax probabilities of the ground-truth rewards and the predicted rewards.

\section{Connection of Bernoulli Backward KL Divergence to DPO}\label{app:rpo-eqs-dpo}
In this subsection we prove the equivalence of Reward-aware Preference Optimization (Bernoulli Backward KL divergence) to \citet{pandey2024brain}, whose special case is DPO \citep{rafailov2024direct}.
\begin{theorem}\label{thm:rpo-bkl-dpo} When using the Bernoulli distribution KL divergence in Reward-aware preference optimization and $\beta=1, \eta=1$, the objective is equivalent to Equation~22 in \citet{pandey2024brain}.
\end{theorem}
\begin{proof} To demonstrate the equivalence, we first recall the definition of the Bernoulli reverse KL divergence:
\begin{align}
  \mathbb{D}\left[a\|b\right] := \mathrm{KL}_{ber}\left[p_b \| p_a\right] = \sigma(b) \log \frac{\sigma(b)}{\sigma(a)} + (1-\sigma(b)) \log \frac{1-\sigma(b)}{1-\sigma(a)} .
\end{align}
Now we explain the Equation~22 in \citet{pandey2024brain}, which is
\begin{align}
    & \expect_{y_1, y_2 \sim p_{ref}(y|x)} \left[ \hat{\alpha}_{y_1} \log \frac{\hat{\alpha}_{y_1} }{\hat{\beta}_{y_1} } + \hat{\alpha}_{y_2} \log \frac{\hat{\alpha}_{y_2} }{\hat{\beta}_{y_2} }  \right], \\
     \hat{\alpha}_{y_1} & = \frac{\exp(r^\star(x, y_1))}{\exp(r^\star(x, y_1)) + \exp(r^\star(x, y_2))} = \frac{1}{1 + \exp( - \left(r^\star(x, y_1) - r^\star(x, y_2))\right) } \notag \\
    &= \sigma\left(r^\star(x, y_1) - r^\star(x, y_2)\right), \notag \\
    \hat{\alpha}_{y_2}&  = \frac{\exp(r^\star(x, y_2))}{\exp(r^\star(x, y_1)) + \exp(r^\star(x, y_2))} = 1 - \hat{\alpha}_{y_1}, \notag  \\
     \hat{\beta}_{y_1}  
     &= \frac{\beta_{y_1}}{\beta_{y_1} + \beta_{y_2}} = \frac{\exp(\log \frac{\pi_{\theta} (y_1 | x) }{\pi_{ref} (y_1 | x)})}{\exp(\log \frac{\pi_{\theta} (y_1 | x) }{\pi_{ref} (y_1 | x)}) + \exp(\log \frac{\pi_{\theta} (y_2 | x) }{\pi_{ref} (y_2 | x)})}  \notag \\
     &= \frac{1}{1 + \exp(-\left(\log \frac{\pi_{\theta} (y_1 | x) }{\pi_{ref} (y_1 | x)} - \log \frac{\pi_{\theta} (y_2 | x) }{\pi_{ref} (y_2 | x)} \right))} \notag \\
     &= \sigma\left(\log \frac{\pi_{\theta} (y_1 | x) }{\pi_{ref} (y_1 | x)} - \log \frac{\pi_{\theta} (y_2 | x) }{\pi_{ref} (y_2 | x)}\right), \notag \\
    \hat{\beta}_{y_2} & = \frac{\beta_{y_2}}{\beta_{y_1} + \beta_{y_2}} = 1 - \hat{\beta}_{y_1} .\notag
\end{align}
Therefore, the objective inside the expectation is a KL divergence between two Bernoulli distributions: one whose probability of 1 is $ \hat{\alpha}_{y_1}$; one whose probability of 1 is $\hat{\beta}_{y_1}$:
\begin{align}
   & \hat{\alpha}_{y_1} \log \frac{\hat{\alpha}_{y_1} }{\hat{\beta}_{y_1} } + \hat{\alpha}_{y_2} \log \frac{\hat{\alpha}_{y_2} }{\hat{\beta}_{y_2} }  = \hat{\alpha}_{y_1} \log \frac{\hat{\alpha}_{y_1} }{\hat{\beta}_{y_1} } + \left(1 - \hat{\alpha}_{y_1}\right) \log \frac{1 - \hat{\alpha}_{y_1} }{1 - \hat{\beta}_{y_1}} \notag \\
    &= \mathrm{KL}_{ber}\left[\sigma(r^\star(x, y_1) - r^\star(x, y_2)) \| \sigma( \log \frac{\pi_{\theta} (y_1 | x) }{\pi_{ref} (y_1 | x)} - \log \frac{\pi_{\theta} (y_2 | x) }{\pi_{ref} (y_2 | x)} ) \right].
\end{align}
\end{proof}

Specifically, when the ground-truth reward margin $r^\star(x, y_1) - r^\star(x, y_2) = \infty$, then $\hat{\alpha}_{y_1} = 1$. Then the RPO objective becomes 
\begin{align}
    - \log \hat{\beta}_{y_1} = - \log \sigma( \log \frac{\pi_{\theta} (y_1 | x) }{\pi_{ref} (y_1 | x)} - \log \frac{\pi_{\theta} (y_2 | x) }{\pi_{ref} (y_2 | x)} ) .
\end{align}
This recovers the DPO's objective.

\section{Offline, Online, and Iterative Algorithms.} \label{app:sec:offline-online-iterative}
We present four algorithm variants of the RPO approach: Offline RPO, Online RPO, Iterative Offline RPO, and Iterative Online RPO. 
\begin{itemize}
    \item \textbf{Offline RPO.} The offline RPO algorithm optimizes the RPO objective with an offline preference dataset whose responses and rewards are pre-computed beforehand.
    \item \textbf{Online RPO.} In online RPO, responses are generated from the online policy $\pi_{\theta}$ and rewards are computed online with $r^\star$. As $\pi_{\theta}$ improves, the quality of its responses improves, which can improve $\pi_{\theta}$ further. Compared to offline RPO, the online RPO is more like RLHF, except that the policy is not trained with reinforcement learning.
    \item \textbf{Iterative Offline RPO.} While online RPO can improve its response quality continuously, it incurs a large computational cost since the response generation online is slow. In addition, the fully online nature of the online RPO makes it vulernerable to reward hacking (see Figure~\ref{fig:training-with-learnt-rm} for detailed discussions). A tradeoff between offline RPO and online RPO is the iterative offline RPO, where we run offline RPO iteratively. In each iteration, the responses are re-sampled from the current policy and the reference policy is updated with the current policy. Then the policy is fully optimized with offline RPO with the curated dataset in the current iteration. Compared to online RPO, since the number of iterations is usually small ($<10$), it is much less likely to have reward hacking. The computational cost is also significantly lower since all the responses are generated offline, which can be done in massive parallelism.
    \item \textbf{Iterative Online RPO.} We can iteratively conduct online alignment as well. In each iteration, we update the reference policy with the latest policy and run online RPO. Compared to the single-iteration online RPO, it enables the model to keep improving itself by evolving the reference policy. On the other hand, it differs with using a smaller $\beta$ in single-iteration online RPO. The difference is that the reference policy used in each iteration are more quality-controlled, which reduces the possibility of reward hacking.
\end{itemize}

\begin{algorithm}[h]
    \caption{Offline RPO.}
    \label{alg:offline-rpo}
    \begin{algorithmic}[1]
        \Require{\textit{Prompt-Response-Reward Dataset}: $\mathcal{D}=\{(x_i, y_i^1, y_i^2, r^{\star}(x, y_i^1), r^{\star}(x, y_i^2))\}_{i=1}^n$.}
        \Require{\textit{Reference Policy}: $\pi_{ref}$; \textit{Distance Metric}: $\mathcal{D}$.}
        \Require{\textit{Hyperparameter}: $\beta$; \textit{Steps}: $T$; \textit{Batchsize}: $B$.}
        \State Initialize $\pi_{\theta} := \pi_{ref}$.
        \For{$t = 1 \; \mathrm{to} \; T$}
            \State Sample batch $\mathcal{B} \subset \mathcal{D}$.
            \State Compute the RPO Objective $\mathcal{L}_{rpo}^{\mathbb{D}}( \pi_{\theta} |\mathcal{B}, r^\star, \pi_{ref}, \beta)$.
            \State Update the policy network $\theta$ with the chosen optimizer.
        \EndFor
        \State \Return $\pi_{\theta}$.
    \end{algorithmic}
\end{algorithm}

\begin{algorithm}[h]
    \caption{Online RPO.}
    \label{alg:online-rpo}
    \begin{algorithmic}[1]
        \Require{\textit{Prompt Dataset}: $\mathcal{D}_x=\{x_i\}_{i=1}^n$.}
        \Require{\textit{Reference Policy}: $\pi_{ref}$; \textit{Distance Metric}: $\mathcal{D}$; \textit{Reward Model}: $r^\star$.}
        \Require{\textit{Hyperparameter}: $\beta$; \textit{Steps}: $T$; \textit{Batchsize}: $B$.}
        \State Initialize $\pi_{\theta} := \pi_{ref}$.
        \For{$t = 1 \; \mathrm{to} \; T$}
            \State Sample the prompts batch $\mathcal{B}_x \subset \mathcal{D}_x$.
            \State Sample two responses from $\pi_{\theta}$ for the prompts in $\mathcal{B}$: $y_i^1, y_i^2 \sim \pi_{\theta}(x_i)$, for $x_i \in \mathcal{B}_x $.
            \State Compute the rewards for the responses: $r^{\star}(x, y_i^1), r^{\star}(x, y_i^2)$.
            \State Construct the batch $\mathcal{B} = \{(x_i, y_i^1, y_i^2, r^{\star}(x, y_i^1), r^{\star}(x, y_i^2)) | x_i \in \mathcal{B}_x\}$.
            \State Compute the RPO Objective $\mathcal{L}_{rpo}^{\mathbb{D}}( \pi_{\theta} |\mathcal{B}, r^\star, \pi_{ref}, \beta)$.
            \State Update the policy network $\theta$ with the chosen optimizer.
        \EndFor
        \State \Return $\pi_{\theta}$.
    \end{algorithmic}
\end{algorithm}

\begin{algorithm}[h]
    \caption{Iterative Offline RPO.}
    \label{alg:iterative-offline-rpo}
    \begin{algorithmic}[1]
        \Require{\textit{Prompt Dataset}: $\mathcal{D}_x=\{x_i\}_{i=1}^n$.}
        \Require{\textit{Reference Policy}: $\pi_{ref}$; \textit{Distance Metric}: $\mathcal{D}$.}
        \Require{\textit{Hyperparameter}: $\beta$; \textit{Iters}: $I$.}
        \State Initialize $\pi_{\theta}^0 := \pi_{ref}$.
        \For{$i = 1 \; \mathrm{to} \; I$} 
            \State Sample prompts $\mathcal{B}_x \subset \mathcal{D}_x$.
            \State Sample two responses from $\pi_{\theta}^{i-1}$ for the prompts in $\mathcal{B}$: $y_i^1, y_i^2 \sim \pi_{\theta}^{i-1}(x_i)$, for $x_i \in \mathcal{B}_x $.
            \State Compute the rewards for the responses: $r^{\star}(x, y_i^1), r^{\star}(x, y_i^2)$.
            \State Construct the training data $\mathcal{D}^i = \{(x_i, y_i^1, y_i^2, r^{\star}(x, y_i^1), r^{\star}(x, y_i^2)) | x_i \in \mathcal{B}_x\}$.
            \State Optimize the policy with the offline RPO using $\mathcal{D}^i$, which returns $\pi_{\theta}^i$.
        \EndFor
        \State \Return $\pi_{\theta}^I$.
    \end{algorithmic}
\end{algorithm}

\begin{algorithm}[h]
    \caption{Iterative Online RPO.}
    \label{alg:iterative-offline-rpo}
    \begin{algorithmic}[1]
        \Require{\textit{Prompt Dataset}: $\mathcal{D}_x=\{x_i\}_{i=1}^n$.}
        \Require{\textit{Reference Policy}: $\pi_{ref}$; \textit{Distance Metric}: $\mathcal{D}$.}
        \Require{\textit{Hyperparameter}: $\beta$; \textit{Iters}: $I$.}
        \State Initialize $\pi_{\theta}^0 := \pi_{ref}$.
        \For{$i = 1 \; \mathrm{to} \; I$} 
            \State Run Online RPO (Algorithm~\ref{alg:online-rpo}) with $\pi_{\theta}^{i-1}$ as the reference policy and initialization, resulting in $\pi_{\theta}^{i}$.
        \EndFor
        \State \Return $\pi_{\theta}^I$.
    \end{algorithmic}
\end{algorithm}

\section{Additional Experiments and Results}\label{app:sec:add-exps}

\subsection{The impact of prompt distributions (Section~\ref{subsec:prompt-impacts})}

We show the average rewards and win-rates over \textit{lmsys (valid)} prompts, \textit{lmsys (test)} prompts, and \textit{alpacaeval} prompts in Table~\ref{tab:in-out-dist-prompts}.
\begin{table*}[h]
\centering
\caption{The average reward and win-rate when training on \textit{lmsys} and \textit{synthetic} prompts, respectively. We compute the metrics over \textit{lmsys (valid)} prompts, \textit{lmsys (test)} prompts, and \textit{alpacaeval} prompts.\label{tab:in-out-dist-prompts}}
\begin{tabular}{ccccccc}
\toprule
                 & \multicolumn{3}{c}{AvgReward}                              & \multicolumn{3}{c}{Win-Rate ($\%$)}                                  \\
                 Prompts & lmsys (valid)      & lmsys (test)      & alpacaeval        & lmsys (valid)      & lmsys (test)        & alpacaeval         \\
\midrule
lmsys            &  $\mathbf{5.498 \pm 0.005}$ & $\mathbf{5.503 \pm 0.016}$ & $5.600 \pm 0.006$ & $\mathbf{71.9 \pm 2.0}$ & $\mathbf{70.5 \pm 1.2}$  & $71.6 \pm 2.0$ \\
synthetic        & $ 5.465 \pm 0.006$ & $5.487 \pm 0.008$ & $\mathbf{5.635 \pm 0.015}$ & $69.6 \pm 0.8$ & $67.3 \pm 1.0$  & $\mathbf{73.3 \pm 1.1}$ \\
\bottomrule
\end{tabular}
\end{table*}



\subsection{Iterative (Offline / Online) Alignment (Section~\ref{sec:iterative})}

We show the average rewards and win-rates over \textit{lmsys (valid)} prompts, \textit{lmsys (test)} prompts, and \textit{alpacaeval} prompts in Table~\ref{tab:iterative-methods}. In additional to iterative offline RPO and iterative online RPO in the 70b setting, we also include iterative offline DPO in the 8b setting, which shows similar learnings.
\begin{table*}[h]
\centering
\caption{The average reward and win-rate in iterative DPO, iterative RPO-bwd (offline), and iterative RPO-bwd (online). We compute the metrics over \textit{lmsys (valid)} prompt, \textit{lmsys (test)} prompt, and \textit{alpacaeval} prompts. \label{tab:iterative-methods}}
\begin{tabular}{cccccccc}
\toprule
                &  & \multicolumn{3}{c}{AvgReward}                              & \multicolumn{3}{c}{Win-Rate ($\%$)}                                  \\
             Base &     Prompts & lmsys (valid)      & lmsys (test)      & alpacaeval        & lmsys (valid)      & lmsys (test)        & alpacaeval         \\
\midrule
\multirow{4}{*}{8b} & SFT            & 5.245  & 5.284 & 5.383 & 50 & 50 & 50 \\
& Offline DPO-Iter-1      & 5.482  & 5.512 &  5.552 & 69.4 & 68.0 & 68.6 \\
& Offline DPO-Iter-2       & 5.569  & 5.588 & 5.743 & 77.4 & 76.1 & 81.2 \\
& Offline DPO-Iter-3       & 5.725  & 5.746 & 5.792  & 85.4 & 83.0 & 84.5 \\
\midrule
\multirow{4}{*}{70b} & SFT            & 5.474  & 5.469 & 5.671 & 50 & 50 & 50 \\
& Offline RPO-bwd-Iter-1      & 5.737  & 5.753 &  5.894 & 77.4 & 75.9 & 79.3 \\
& Offline RPO-bwd-Iter-2       & 5.878  & 5.875 & 5.98 & 88.5 & 85.2 & 88.1 \\
& Offline RPO-bwd-Iter-3      & 5.970  &  5.981 & 6.051 & 92.9 & 90.3 & 92.2 \\
\midrule
\multirow{4}{*}{70b} & SFT            & 5.474  & 5.469 & 5.671 & 50 & 50 & 50 \\
& Online RPO-bwd-Iter-1      & 5.883  & 5.916 &  5.992 & 82.1 & 85.7 & 85.5 \\
& Online RPO-bwd-Iter-2       & 5.985 & 5.986  & 6.060  & 88.5  & 88.3 & 88.0  \\
& Online RPO-bwd-Iter-3      & 6.061 & 6.056  & 6.157  & 91.2  & 90.1  & 93.8 \\
\bottomrule
\end{tabular}
\end{table*}

\section{Experimental Details} \label{app:sec:details}

\paragraph{Evaluations.} We use the average rewards over \textit{lmsys (valid)} prompts to select the best checkpoint and the best hyper-parameter. We then repeat the model training with the best hyper-parameter multiple times with randomly shuffled data (only in the 8b setting for the sake of computational costs). For each random training, we evaluate the validation and test metrics over the best checkpoint. Then we report the average metric and 95\% confidence interval.

\paragraph{Offline preference optimization methods.} For all offline RPO training jobs (8b or 70b, K=2 or K=4), we use the Adam optimizer with a constant learning rate 5e-7. We set the batch size at 256 and train the model for one epoch. We save checkpoints every 50 iterations, best of which is selected according to the validation average reward. Similarly for KTO, we use Adam lr of 5e-7 for 8b and 7e-7 for 70b both with batch size of 256. KL regularization used was $1e\text{-}2$ for 8b and $6e\text{-}2$ for 70b. Checkpoint selection was done similar to offline RPO. For RPO-bwd, We tune the KL regularization coefficient $\beta \in [1e\text{-}3, 1e\text{-}2]$ and the explicit RM coefficient $\eta \in [1, 100]$. For RPO-sqloo, we tune the KL regularization coefficient $\beta \in [3e\text{-}3, 1e\text{-}1]$ and fix $\eta=1$ since increasing $\eta$ is equivalent to poportionally decreasing $\beta$ in RPO-sqloo.

\paragraph{Online preference optimization methods.} For online RPO training jobs (70b, K=4), we use the Adam optimization with a constant learning rate 5e-7. We set the global batch size at 64 and train the model for as far as 250 iterations (13\% of one epoch) due to computational limitations. We save checkpoints every 10 iterations, best of which is selected according to the validation average reward. For RPO-bwd, we tune the KL regularization coefficient $\beta \in [3e\text{-}3, 1e\text{-}2]$ and the explicit RM coefficient $\eta \in [10, 30]$. For RPO-sqloo with 70B models (i.e., RLOO), we tune the KL regularization coefficient $\beta \in [1e\text{-}3, 3e\text{-}2]$ and fix $\eta=1$ since increasing $\eta$ is equivalent to poportionally decreasing $\beta$ in RPO-sqloo. While multiple RPO-sqloo training runs crashed in the middle of training, all RPO-bwd training runs stably increased. For 8B models we tune the learning rate in $[2e\text{-}7, 8e\text{-}7]$ and tune the KL regularization coefficient $\beta \in [1e\text{-}4, 1e\text{-}2]$.

\paragraph{Training the Learnt RM.} We use the preference dataset to train a reward model for 8b and 70b, respectively. The reward model is initialized with the SFT checkpoint and trained with a constant learning rate for one epoch. We used lr=$\text{3e-6}$ and batch size 512 for training the 8B RM; we use lr $\text{1e-6}$ and batch size $256$ for the 70B RM. For the 8B RM, we used the Bradley-Terry model~\citep{ouyang2022training} for the training objective; for the 70b RM, we used the regression model~\citep{wang2024helpsteer2}. 

\paragraph{Iterative alignment.} The preference optimization dataset used in Iter-1 contains 120k unique prompts. To enable iterative alignment, we additionally collected 150k lmsys prompts, which are disjointed from these preference optimization prompts and the SFT prompts. We selected 100k randomly prompts out of the 150k new prompts and built the preference datasets in Iter-2. We trained the models (offline DPO, offline RPO) for one epoch. Then we mixed the remaining 50k prompts and another random 50k prompts from the Iter-1 preference datasets to build the preference dataset in Iter-3. We then trained models (offline DPO, offline RPO) for one epoch.

\end{document}